\documentclass{article}

\usepackage[main,final]{neurips_2026}

\usepackage[utf8]{inputenc} 
\usepackage[T1]{fontenc}    
\usepackage[colorlinks=true,linkcolor=blue,citecolor=blue,urlcolor=blue]{hyperref}      
\usepackage{xurl}            
\usepackage{booktabs}       
\usepackage{amsfonts}       
\usepackage{nicefrac}       
\usepackage{microtype}      
\usepackage{xcolor}         

\usepackage{amsmath,amssymb,mathtools}
\usepackage{amsthm}
\usepackage{physics}
\mathtoolsset{showonlyrefs}  

\usepackage{graphicx}
\usepackage{subcaption}   
\usepackage{float}
\usepackage{wrapfig}
\usepackage{tikz}

\usepackage{csquotes}
\usepackage{etoolbox}
\usepackage{makecell}
\usepackage{enumitem}     
\usepackage{pifont}
\usepackage{soul}         
\usepackage{cancel}
\usepackage{bbm}
\usepackage{mathrsfs}

\usepackage[ruled]{algorithm2e}
\DontPrintSemicolon

\theoremstyle{plain}
\newtheorem{theorem}{Theorem}[section]
\newtheorem{lemma}[theorem]{Lemma}
\newtheorem{example}[theorem]{Example}

\theoremstyle{definition}
\newtheorem{definition}[theorem]{Definition}
\newtheorem{assumption}[theorem]{Assumption}
\theoremstyle{remark}

\newcommand{\fS}{\mathcal{S}}\newcommand{\fA}{\mathcal{A}}

\newcommand{\fO}{\mathcal{O}}

\newcommand{\R}{\mathbb{R}}\newcommand{\E}{\mathbb{E}}

\title{On the Divergence of Differential Temporal Difference Learning without Local Clocks}

\author{%
David Antrobius \\
  Department of Computer Science\\
  University of Virginia\\
  \texttt{davidantrobius@email.virginia.edu} \\
  \And
  Shangtong Zhang \\
  Department of Computer Science\\
  University of Virginia\\
  \texttt{shangtong@virginia.edu} \\
}

\begin{document}

\maketitle

\begin{abstract}
    Learning rate is a critical component of reinforcement learning (RL).
    This work uses global and local clocks to distinguish two types of learning rates.
    The former is of the standard form $\alpha_t$ that depends only on the time step $t$ (i.e., a global clock).
    The latter is of the form $\alpha_{\nu(S_t, t)}$,
    where $\nu(s, t)$ counts the number of visits to state $s$ until time $t$ (i.e., a local clock).
    In discounted RL, an RL algorithm that is convergent with a local clock is always also convergent with a global clock, and vice versa.
    We are not aware of any counterexample.
    The key contribution of this work is to show that this nice correspondence breaks down in average-reward RL.
    Specifically,
    we construct a counterexample showing that although differential temporal difference (TD) learning is convergent with a local clock, it can diverge with a global clock.
    We empirically observe this divergence in our experiments and also prove that the associated ODE of differential TD with a global clock is unstable.
    This counterexample closes the open problem in \citet{wan2020learning,blaser2026almost}.
\end{abstract}

\section{Introduction}

Tabular methods form a fundamental class of algorithms in reinforcement learning (RL, \citet{sutton2018reinforcement}).
Looking back at the history of tabular RL, 
we can see a clear pattern that early theoretical analysis of tabular RL methods often uses a local clock in the learning rate.
Here, by a local clock in a learning rate, we mean a learning rate of the form $\alpha_{\nu(S_t, t)}$, where $\nu(s, t)$ counts the number of visits to state $s$ until time $t$.
This is in contrast to a learning rate with a global clock, which is of the form $\alpha_t$ that depends only on the time step $t$.
The global clock refers to the time step $t$ in the learning rate, while the local clock refers to the state-dependent visit count $\nu(s, t)$ in the learning rate. 
For example,
classical convergence proofs for $Q$-learning usually use local clocks \citep{watkins1989learning,watkins1992q,jaakkola1993convergence,tsitsiklis1994asynchronous,bertsekas1996neuro}.

We argue that the use of local clocks in theoretical analysis is mostly because of technical convenience.
Chapter~7 of \citet{borkar2009stochastic} gives a thorough exposition of the role of local clocks in the convergence analysis of stochastic approximation algorithms.
Intuitively, local clocks allow more balanced updates across states, essentially converting the Markovian samples in RL into i.i.d. samples,
which significantly simplifies the analysis.
However, a local clock is not really a practitioner's first choice.
For example,
the entire textbook \citet{sutton2018reinforcement} does not mention or use any local clock in the learning rates.
A global clock is preferred by most practitioners.

This theorist-practitioner gap is not a problem in the discounted RL setting.
Although early theoretical analysis of discounted RL methods often uses local clocks, 
recent works have shown that the same algorithms are also convergent with global clocks.
For example, the convergence of $Q$-learning with a global clock is also established by \citet{lee2019unified,chen2024lyapunov,liu2025ode}. 
To our knowledge, there is no known counterexample showing that an algorithm that is convergent with a local clock can diverge with a global clock in discounted RL.

In average-reward RL, 
early theoretical analysis such as RVI $Q$-learning also uses local clocks in the learning rates \citep{abounadi2001learning}.
The seminal work \cite{wan2020learning} 
develops a new set of temporal difference algorithms for average-reward RL that removes the need for a reference state in RVI $Q$-learning.
The (off-policy) policy evaluation algorithm in \citet{wan2020learning} is called differential TD.
And the control algorithm in \citet{wan2020learning} is called differential $Q$-learning.
In the main text of \citet{wan2020learning}, both algorithms are presented with a global clock in the learning rates.
However,
in the convergence analysis in the appendix, both algorithms are shown to be convergent only with a local clock in the learning rates.
So the open problem is, do these algorithms also converge with a global clock in the learning rates?
\citet{blaser2026almost} partially answer this question by showing that when $\eta$ is sufficiently small, differential TD is convergent even with a global clock.
Here $\eta$ is a hyperparameter in differential TD that we will introduce later.
However, there is still a gap because \citet{wan2020learning} show that with a local clock, differential TD is convergent for any $\eta > 0$.
Does differential TD also converge with a global clock for any $\eta > 0$?
This is an open problem from \citet{blaser2026almost}.

One may optimistically conjecture that the answer to the above question is yes.
After all in the discounted RL setting, we have seen that algorithms that are convergent with a local clock are also convergent with a global clock.
We are not aware of any counterexample in the literature.
The key contribution of this work is to show that this nice correspondence breaks down in average-reward RL.
Particularly, we construct a counterexample showing that differential TD can diverge with a global clock.
Empirically, we observe this divergence in our experiments.
Theoretically, we prove that the ODE associated with differential TD with a global clock can be unstable.
This negative result is surprising and interesting because it shows that the choice of learning rates can have a much more significant impact on the convergence of RL algorithms in the average-reward setting than in the discounted setting.

Our contribution is beyond the counterexample itself.
Specifically, we provide a detailed analysis of the stability of the ODE associated with differential TD with a global clock.
Beyond the sufficient small-$\eta$ condition of \citet{blaser2026almost}, we characterize the first loss of positive stability and the structure of the full stability region for that ODE.
Our counterexample shows that this region can be disconnected, with stability lost and then recovered as $\eta$ increases.

\section{Background}

We consider an infinite horizon Markov Decision Process (MDP, \citet{bellman1957markovian}) defined by a tuple $(\fS, \fA, p, p_0, r)$, where $\fS$ is the finite state space, $\fA$ is the finite action space, $p: \fS \times \fS \times \fA \to [0, 1]$ is the transition function, $p_0: \fS \to [0, 1]$ is the initial state distribution, and $r: \fS \times \fA \to \R$ is the reward function.
At time step $0$, an initial state $S_0$ is drawn from $p_0$.
At each time step $t \geq 0$, 
an action $A_t$ is chosen according to a policy $\pi: \fA \times \fS \to [0, 1]$ as $A_t \sim \pi(\cdot | S_t)$.
Then a reward $R_{t+1} = r(S_t, A_t)$ is emitted and the next state $S_{t+1}$ is drawn from $p(\cdot | S_t, A_t)$.

In the discounted RL setting, we consider a discount factor $\gamma \in [0, 1)$.
The goal of policy evaluation is to estimate the value function for a given policy $\pi$, defined as $v_\pi(s) \doteq \E\qty[\sum_{i=0}^\infty \gamma^i R_{t+i+1} | S_t = s]$.
In the average-reward RL setting,
we consider the average reward $J_\pi \doteq \lim_{T \to \infty} \frac{1}{T} \E\qty[\sum_{t=0}^{T-1} R_{t+1}]$ and are interested in the differential value function of a policy $\pi$, 
defined as $\bar v_\pi(s) \doteq \E\qty[\sum_{i=0}^\infty (R_{t+i+1} - J_\pi) | S_t = s]$.
 We assume that the Markov chain induced by the target policy is irreducible, so $J_\pi$ is independent of  the initial state.

Temporal difference (TD) learning is the most fundamental class of algorithms for policy evaluation in RL.
In this work,
we consider off-policy TD,
where the data is generated by a behavior policy $\mu$ (i.e., $A_t \sim \mu(\cdot | S_t)$) while we want to estimate the value function for a different target policy $\pi$.
Let $\{\alpha_n\}_{n\geq0}$ denote the deterministic base learning-rate sequence.
In the discounted setting, the TD update is given by
\begin{align}
  \label{eq:discounted_td global}
  \delta_t =& R_{t+1} + \gamma v_t(S_{t+1}) - v_t(S_t), \\
  v_{t+1}(s) =& \begin{cases}
    v_t(s) + \alpha_t \rho_t \delta_t, & \text{if } s = S_t, \\
    v_t(s), & \text{otherwise}
  \end{cases},
\end{align}
where 
$\rho_t = \frac{\pi(A_t | S_t)}{\mu(A_t | S_t)}$ is the importance sampling ratio.
The vector $v_t\in\R^{|\fS|}$ estimates the discounted value function $v_\pi$.
Here, $\alpha_t$ is the learning rate, e.g., $\alpha_t = \frac{1}{t+1}$ and we refer to~\eqref{eq:discounted_td global} as \emph{discounted TD with a global clock}.
Let $\nu(s, t) \doteq \sum_{i=0}^t \mathbb{I}_{\{S_i = s\}}$ be the number of times state $s$ has been visited by time $t$ with $\mathbb{I}$ being the indicator function.
If we replace $\alpha_t$ with $\alpha_{\nu(S_t, t)}$ in~\eqref{eq:discounted_td global}, we get \emph{discounted TD with a local clock}.
Both clocks use asynchronous value updates: at each time step, only the component corresponding to the visited state is updated.
 Discounted TD converges with either clock under the corresponding chain, coverage, and learning-rate conditions stated in Section~\ref{sec:clock_odes} and Appendix~\ref{app:convergence_assumptions}.

\citet{wan2020learning} extend off-policy TD to the average-reward setting and propose the following update for differential TD:
\begin{align}
  \label{eq:differential_td global}
  \delta_t =& R_{t+1} - \hat J_t + v_t(S_{t+1}) - v_t(S_t), \\
  \hat J_{t+1} =& \hat J_t + \alpha_t \eta \rho_t \delta_t, \\
  v_{t+1}(s) =& \begin{cases}
    v_t(s) + \alpha_t \rho_t \delta_t, & \text{if } s = S_t, \\
    v_t(s), & \text{otherwise}
  \end{cases},
\end{align}
where $\eta > 0$ is a hyperparameter and $\hat J_t$ is an estimate of the average reward $J_\pi$. The vector $v_t$ estimates the differential value function $\bar v_\pi$ up to an additive constant. The learning rate for $\hat J_t$ is $\eta$ times that for $v_t(S_t)$.
We refer to~\eqref{eq:differential_td global} as \emph{differential TD with a global clock}.
Similarly,
by replacing both appearances of $\alpha_t$ in~\eqref{eq:differential_td global} with $\alpha_{\nu(S_t, t)}$, we get \emph{differential TD with a local clock}.
\citet{wan2020learning} prove the almost sure convergence of differential TD with a local clock for any $\eta > 0$ in their appendix under the assumptions summarized in Appendix~\ref{app:convergence_assumptions},
although their main text only presents differential TD with a global clock and does not mention the local clock at all.
\citet{blaser2026almost} show that when $\eta$ is sufficiently small, differential TD with a global clock is also convergent under their stated assumptions; Appendix~\ref{app:convergence_assumptions} states sufficient conditions used there.
It is thus an open problem whether differential TD with a global clock is convergent for an $\eta$ that is not small enough.
Surprisingly, we will show in this paper that the answer is negative:
the divergence of the algorithm is empirically observed and the instability of the associated ODE is proved.

We now introduce some definitions for matrix analysis.
\begin{definition}
  A matrix $A$ is called an $M$-matrix if $A$ can be expressed as $A = sI - B$, where $s > 0$, $B \geq 0$ (i.e., all entries of $B$ are nonnegative), and $s \geq \rho(B)$ with $\rho(B)$ being the spectral radius of $B$.
\end{definition}
\begin{definition}
  A matrix $A$ is called positive stable if all eigenvalues of $A$ have strictly positive real parts.
  A matrix $A$ is Hurwitz if $-A$ is positive stable.
\end{definition}

\section{How Is Differential TD Different from Discounted TD?}
\label{sec:clock_odes}
We now provide some necessary background to illustrate the fundamental difference between how discounted and differential TD interact with global and local clocks. 
This is largely due to \citet{blaser2026almost} and we provide details here only for completeness and to set the stage for our counterexample in the next section.
For the discussion below, we assume that the target and behavior chains are irreducible and that $\pi(a|s)>0$ implies $\mu(a|s)>0$ for every state-action pair.

The ODE method is a powerful tool for establishing the almost sure convergence of RL algorithms.
Take discounted TD with a local clock as an example.
To apply the ODE method, we study the ODE$@\infty$ associated with discounted TD with a local clock, which is given by
\begin{align}
  \dv{x(t)}{t} =& -(I - \gamma P_\pi) x(t),
\end{align}
where $I$ is the identity matrix and $P_\pi$ is the transition matrix under policy $\pi$, i.e., $P_\pi(s, s') = \sum_{a} \pi(a|s) p(s'|s, a)$.
The standard ODE method (e.g., \citet{borkar2009stochastic,borkar2025ode,liu2025ode}) needs to verify that the above ODE is globally asymptotically stable (GAS) and 0 is its unique equilibrium point.
Fortunately, it is the case here because $I - \gamma P_\pi$ is a nonsingular $M$-matrix. It is well-known that a nonsingular $M$-matrix is positive stable \citep{Plemmons1977MmatrixCM}. 
The ODE$@\infty$\footnote{Here, ODE$@\infty$ denotes the ODE at infinity technique used in the Borkar-Meyn stability framework \citep{borkar2000ode}.} associated with discounted TD with a global clock is instead
\begin{align}
  \dv{x(t)}{t} =& -D_\mu (I - \gamma P_\pi) x(t),
\end{align}
where $D_\mu$ is a diagonal matrix with $D_\mu(s, s) = d_\mu(s)$ and $d_\mu$ is the stationary distribution of the Markov chain induced by $\mu$.
Suppose the Markov chain induced by $\mu$ is irreducible, then $d_\mu(s) > 0$ for all $s$.
So $D_\mu$ is also positive stable.
It's also known that the left multiplication of a nonsingular $M$-matrix by a positive stable diagonal matrix is still positive stable \citep{Fiedler1962,Plemmons1977MmatrixCM}, which implies that $D_\mu (I - \gamma P_\pi)$ is also positive stable and thus the ODE$@\infty$ associated with discounted TD with a global clock is also GAS.
The difference between the two ODEs is intuitive.
With a local clock, all the states are updated at the same ``magnitude'',
so the actual sampling distribution does not matter.
But with a global clock, the states are updated at different magnitudes according to how frequently they are visited.
So there is a $D_\mu$ term to capture the sampling distribution.

For differential TD with a local clock, the corresponding ODE$@\infty$ is 
\begin{align}
  \dv{x(t)}{t} = -(I - P_\pi + \eta ee^\top) x(t),
\end{align}
where $e$ is the all-one vector.
The detailed derivation can be found in \citet{wan2020learning}.
Although $(I - P_\pi + \eta ee^\top)$ is in general not an $M$-matrix,
it is still positive stable for all $\eta > 0$.
So the ODE is still GAS.
The ODE$@\infty$ associated with differential TD with a global clock is instead
\begin{align}
  \dv{x(t)}{t} = -D_\mu (I - P_\pi + \eta ee^\top) x(t).
\end{align}
See \citet{blaser2026almost} for the detailed derivation.
Unfortunately, the left multiplication of a positive stable matrix by a positive stable diagonal matrix does not necessarily preserve the positive stability.
So the ODE$@\infty$ associated with differential TD with a global clock is not automatically GAS for all $\eta > 0$.
To summarize, the fundamental difference is that \textbf{discounted TD has an $M$-matrix but differential TD has only a positive stable matrix}.

\citet{blaser2026almost} approach this challenge via a rank-one perturbation perspective.
Let 
\begin{align}
  A_\eta \doteq D_\mu (I - P_\pi + \eta ee^\top) = D_\mu (I - P_\pi) + \eta d_\mu e^\top.
\end{align}
Then $A_\eta$ can be viewed as a rank-one perturbation of a singular M-matrix because $D_\mu (I - P_\pi)$ is a singular M-matrix (see \citet{blaser2026almost} for a simple proof) and $\eta d_\mu e^\top$ is a rank-one matrix.
With results from the rank-one perturbation community \citep{bierkens2014singular}, \citet{blaser2026almost} show that when $\eta$ is sufficiently small, $A_\eta$ is positive stable  and thus differential TD with a global clock is convergent under the assumptions stated in Appendix~\ref{app:convergence_assumptions}.

\section{Counterexamples}
\label{sec:counterexamples}
In this section,
we provide a full characterization about when $A_\eta$ is positive stable,
based on which we construct concrete counterexamples showing that $A_\eta$ is not positive stable for some $\eta > 0$.
The following assumption is standard and is used in all the rest of the paper.
\begin{assumption}
\label{ass:matrix_conditions}
The target transition matrix $P_\pi$ is irreducible and $ d_\mu(s) > 0$ for all $s$.
\end{assumption}
We follow the rank-one perturbation perspective in \citet{blaser2026almost} and first study
\begin{align}
  L \doteq D_\mu(I - P_\pi).
\end{align}
We use $\operatorname{spec}(L) \doteq \{z\in\mathbb C: \det(zI-L)=0\}$ to denote the spectrum of $L$, i.e., the set of its eigenvalues.
Our first observation is 
that every eigenvalue of $L$ has nonnegative real part and every nonzero eigenvalue has strictly positive real part.
\begin{lemma}
  \label{lemma:spectrum}
  For any $\lambda \in \operatorname{spec}(L)$, $\Re(\lambda) \geq 0$. If $\lambda \neq 0$, then $\Re(\lambda) > 0$.
\end{lemma}
\begin{proof}
  Let $\lambda\in\operatorname{spec}(L)$ and let $x\neq0$ satisfy $Lx=\lambda x$.  Choose $k$ such that $|x_k|=\max_j |x_j|$.  The $k$-th row of $Lx=\lambda x$ gives $\textstyle d_\mu(k)\left(x_k-\sum_j P_\pi(k,j)x_j\right)=\lambda x_k,$ so $\textstyle \sum_j P_\pi(k,j)x_j = \left(1-\frac{\lambda}{d_\mu(k)}\right)x_k.$ Since $P_\pi$ is row-stochastic, $$\textstyle \left|1-\frac{\lambda}{d_\mu(k)}\right| |x_k| \leq \sum_jP_\pi(k,j)|x_j| \leq |x_k|.$$ Thus $|1-\lambda/d_\mu(k)|\leq1$.  Writing $\lambda=a+ib$ yields $\left(1-\frac{a}{d_\mu(k)}\right)^2 + \left(\frac{b}{d_\mu(k)}\right)^2 \leq 1.$
  By rearranging terms, we have $2d_\mu(k)\operatorname{Re}(\lambda)\geq |\lambda|^2,$
  which completes the proof.
\end{proof}
Lemma~\ref{lemma:spectrum} clarifies that \emph{0 is the only eigenvalue of $L$ on the imaginary axis and all the other eigenvalues of $L$ are in the open right half plane}.
Now, define a function $g: \mathbb{C}\setminus\operatorname{spec}(L) \to \mathbb{C}$ as
\begin{align}
  g(z) \doteq e^\top (zI - L)^{-1} d_\mu.
\end{align}
It is easy to see that $g(z)$ is well-defined for every $z \notin \operatorname{spec}(L)$.
Particularly, let $i$ be the imaginary unit. Then for every $\omega > 0$, $i \omega \notin \operatorname{spec}(L)$ thanks to Lemma~\ref{lemma:spectrum}. 
So $g(i \omega)$ is well-defined.
This allows us to define
\begin{align}
  \eta_* \doteq \inf \qty{\eta > 0 \mid \exists \omega > 0 \text{ such that } 1 - \eta g(i\omega) = 0}.
\end{align}
with the convention that $\inf \emptyset = +\infty$. 
This $\eta_*$ is the threshold we want.
\begin{theorem}[Maximal Stability Threshold]
  \label{thm:maximal_stability_threshold}
  For every $\eta \in (0, \eta_*)$, $A_\eta$ is positive stable.
  Furthermore, the above statement does not hold for any $\eta_*' > \eta_*$, i.e., for every $\eta_*' > \eta_*$, there exists some $\eta_c \in [\eta_*, \eta_*')$ such that $A_{\eta_c}$ is not positive stable.
\end{theorem}
\begin{proof}
  We view $A_\eta$ as a rank-one perturbation of $L$ and study how the eigenvalues of $A_\eta$ move as $\eta$ increases from $0$.
  Since eigenvalues depend continuously on the entries of the matrix, they also depend continuously on $\eta$.
  By Lemma~4.9 of \citet{blaser2026almost}, there exists $\eta_0>0$ such that $A_\eta$ is positive stable for every $\eta\in(0,\eta_0]$.\footnote{Lemma~4.9 of \citet{blaser2026almost} proves this using Lemma~2.11 of \citet{bierkens2014singular}. For completeness, Appendix~\ref{app:small_eta_stability} gives a self-contained perturbation proof for our setting.}
  In other words, all the eigenvalues of $A_\eta$ are in the open right half plane for all sufficiently small $\eta > 0$.
  Thus for $A_\eta$ to lose positive stability when $\eta$ increases, at least one eigenvalue must cross the imaginary axis for some $\eta > 0$.

  First, we will show that it cannot cross the imaginary axis through the origin. 
  This is because $A_\eta$ has no zero eigenvalue for any $\eta > 0$.
  To see this, let $d_\pi$ be the stationary distribution of $P_\pi$ and define $\ell^\top \doteq d_\pi^\top D_\mu^{-1}$.
  If $A_\eta x = 0$ for some $x \neq 0$ and $\eta > 0$, 
  then
  \begin{align}
    0 = \ell^\top A_\eta x = \ell^\top D_\mu (I - P_\pi) x + \eta \ell^\top d_\mu e^\top x = \eta e^\top x.
  \end{align}
  Then 
    $0 = A_\eta x =Lx + \eta d_\mu e^\top x = Lx$.
  So $x$ is in the kernel of $L$.
  Since $D_\mu$ is full rank,
  the kernel of $L$ is the same as the kernel of $I - P_\pi$.
  But by the Perron-Frobenius theorem, the kernel of $I - P_\pi$ is one-dimensional and is spanned by the all-one vector $e$.
  So there must exist $c \neq 0$ such that $x = ce$,
  yielding $0 = e^\top x = ce^\top e \neq 0$, which is a contradiction.

  Now we characterize every possible non-origin imaginary axis crossing.
  Notice that $A_\eta$ is a real matrix, so nonreal eigenvalues of $A_\eta$ come in conjugate pairs.
  It is thus sufficient to characterize every possible crossing at $i \omega$ with $\omega > 0$.
  We thus study when $i \omega$ is an eigenvalue of $A_\eta$ for some $\omega > 0$ and $\eta > 0$, i.e., when $\det(i \omega I - A_\eta) = 0$.
  By the rank-one determinant identity $\det(M+uv^\top)=\det(M)+v^\top\operatorname{adj}(M)u$ with $\operatorname{adj}(M)$ being the adjugate of $M$,
  we have
  \begin{align}
    0 = \det(i\omega I-A_\eta) =& \det(i\omega I-L-\eta d_\mu e^\top) = \det(i\omega I-L)-\eta e^\top\operatorname{adj}(i\omega I-L)d_\mu.
  \end{align}
  Since $i\omega \notin \operatorname{spec}(L)$, we have $\operatorname{adj}(i\omega I-L)=\det(i\omega I-L)(i\omega I-L)^{-1}$, so
  \begin{align}
    0 =& \det(i\omega I-L)-\eta e^\top\operatorname{adj}(i\omega I-L)d_\mu = \det(i\omega I-L)(1-\eta g(i\omega)).
  \end{align}
  Again since $i\omega \notin \operatorname{spec}(L)$, 
  we have $\det(i\omega I-L)\neq0$.
  So the above is equivalent to $1-\eta g(i\omega)=0$.
  By definition, $\eta_*$ is the infimum of such $\eta$,
  which completes the proof.
\end{proof}

Theorem~\ref{thm:maximal_stability_threshold} gives the largest $\eta_*$ such that the claim ``$A_\eta$ is positive stable for every $\eta \in (0, \eta_*)$'' holds.
In the following, we show that the stability region can be more complicated than just $(0, \eta_*)$. 
\begin{theorem}[Exact Stability Region]
  \label{thm:exact_stability_region}
  The stability region 
  $$\Pi \doteq \qty{\eta > 0 \mid A_\eta \text{ is positive stable}}$$
  is always a finite union of open intervals, possibly including unbounded intervals.
\end{theorem}

\begin{proof}
To characterize $\Pi$, we will use the Routh-Hurwitz criterion \citep{gantmacher1980theory}.
To this end, we consider the characteristic polynomial of $-A_\eta$ defined as
\begin{align}
  q_\eta(z) \doteq \det(zI + A_\eta) = z^m + a_1(\eta) z^{m-1} + \cdots + a_m(\eta),
\end{align}
where $m = |\fS|$ is the number of states and $\qty{a_j(\eta)}_{j=1}^m$ are the coefficients of the characteristic polynomial.
The roots of $q_\eta(z)=\det(zI+A_\eta)$ are exactly $z=-\lambda$, where $\lambda\in\operatorname{spec}(A_\eta)$. 
Therefore $A_\eta$ is positive stable if and only if every root of $q_\eta$ has negative real part, i.e., if and only if $q_\eta$ is a Hurwitz polynomial. 
By the Routh-Hurwitz criterion, a real monic polynomial is Hurwitz if and only if all of its Hurwitz determinants are positive. 
Using the convention that $a_0(\eta) \doteq 1$ and $a_j(\eta) \doteq 0$ whenever $j < 0$ or $j > m$, 
for $k = 1, \dots, m$,
the $k$-th Hurwitz determinant of the polynomial $q_\eta$ is defined as
\begin{align}
  \Delta_k(\eta) \doteq \det\left(\left[a_{2j-i}(\eta)\right]_{i,j=1}^k\right).
\end{align}
Thus $A_\eta$ is positive stable if and only if $\Delta_k(\eta)>0$ for every $k=1,\ldots,m$.
Since $A_\eta=L+\eta d_\mu e^\top$, we have
\begin{align}
  q_\eta(z) = \det(zI+L+\eta d_\mu e^\top) = \det(zI+L) + \eta e^\top \operatorname{adj}(zI+L) d_\mu.
\end{align}
Hence $q_\eta(z)=q_0(z)+\eta q_1(z)$, where $q_0$ and $q_1$ are polynomials in $z$. Consequently, each coefficient $a_j(\eta)$ must be linear in $\eta$.
Since $\Delta_k(\eta)$ is the determinant of a $k\times k$ matrix whose entries are linear in $\eta$, it is a polynomial in $\eta$ of degree at most $k$. 
If some $\Delta_k(\eta)$ is zero for all $\eta > 0$, then no $\eta$ can satisfy $\Delta_k(\eta)>0$, so the stability region is empty. 
Otherwise, the finitely many real roots of the nonzero polynomials $\Delta_1,\ldots,\Delta_m$ partition $(0,\infty)$ into finitely many intervals, and the sign of each $\Delta_k$ is constant on each such interval. Selecting exactly those intervals on which all determinants are positive gives $\Pi$. 
Thus the stability region is a finite union of open intervals, possibly empty and possibly including unbounded intervals.
\end{proof}

We note that Theorems~\ref{thm:maximal_stability_threshold} \&~\ref{thm:exact_stability_region} do not rule out the possibility that $\Pi$ is $(0, \infty)$.
Rather, they provide a characterization of what $\Pi$ may look like.
Building on these insights, we will now construct a concrete counterexample showing that for some $\eta > 0$, $A_\eta$ is not positive stable.
The construction combines a cyclic target-policy structure through three groups of states with unequal behavior-state visitation frequencies.
Its symmetry isolates a three-dimensional subspace in which the rank-one term $\eta d_\mu e^\top$ can destabilize the dynamics, allowing an exact calculation of where stability is lost and recovered.

\begin{example}[Counterexample]
  \label{example:counterexample}
  Suppose $m>22$ is an integer. 
  Then there exists a positive diagonal matrix $D_\mu \in \R^{(m + 2)\times (m + 2)}$ and an irreducible and aperiodic row-stochastic matrix $P_\pi$ such that
  $A_\eta$ is positive stable if and only if
  \begin{align}
    \label{eq:stability_region}
  \eta\in(0,\varepsilon)\cup(3\varepsilon,\infty), \qquad \varepsilon\doteq \frac{m-22}{m^2+m-2}.
  \end{align}
  The smallest instance in this family is obtained at $m=23$, yielding $\varepsilon=1/550$ and a $25$-state system.
\end{example}

\begin{proof}
  Consider $m+2$ states $a_1,\ldots,a_m,b,c$. 
  Set $\varepsilon\doteq(m-22)/(m^2+m-2)$.
  Construct a vector $d_\mu \in \R^{m+2}$ as $d_\mu[a_i]=\varepsilon$ for $i=1,\ldots,m$, $d_\mu[b]=\varepsilon$, and $d_\mu[c]=1-(m+1)\varepsilon$, and let $D_\mu \doteq \operatorname{diag}(d_\mu)$. 
  Since $m>22$, we have $\varepsilon>0$. Moreover,
  $$d_\mu[c]=\frac{22m+20}{m^2+m-2}>0, \qquad
    d_\mu[c]-\varepsilon=\frac{21m+42}{m^2+m-2}>0.$$
  Also, $\sum_s d_\mu[s]=m\varepsilon+\varepsilon+d_\mu[c]=1$. Thus $d_\mu$ is a strictly positive probability vector, and $\varepsilon\le d_\mu[s]$ for every state $s$.

  Define the row-stochastic matrix $Q$ by
  $$Q[a_i, b]=1\ (i=1,\ldots,m), \qquad Q[b, c]=1, \qquad Q[c, a_i]=\frac1m\ (i=1,\ldots,m),$$
  with all other entries zero.
  Define $P_\pi\doteq I+\varepsilon D_\mu^{-1}(Q-I)$. Then, for each state $s$, $$\textstyle P_\pi[s,\cdot]=\left(1-\frac{\varepsilon}{d_\mu[s]}\right)e_s^\top+\frac{\varepsilon}{d_\mu[s]}Q[s,\cdot],$$ where $e_s$ is the standard basis vector at state $s$, not the all-one vector $e$. 
  Therefore every row of $P_\pi$ is a convex combination of $e_s^\top$ and $Q[s,\cdot]$, so $P_\pi$ is row-stochastic. It is irreducible because all transitions of $Q$ remain available, and it is aperiodic because $P_\pi[c,c]=1-\varepsilon/d_\mu[c]>0$.

  For $\eta>0$, write $t\doteq\eta/\varepsilon$. Then $$A_\eta = D_\mu(I-P_\pi+\eta ee^\top)=\varepsilon(I-Q+td_\mu e^\top),$$ since $D_\mu(I-P_\pi)=\varepsilon(I-Q)$ and $D_\mu\eta ee^\top=\eta d_\mu e^\top=\varepsilon td_\mu e^\top$. 
  Since $\varepsilon>0$, multiplication by $\varepsilon$ does not change positive stability. 
  So it is sufficient to analyze $B_t\doteq I-Q+td_\mu e^\top$.

  We next decompose the ambient vector space. Work over $\mathbb C$ and order the coordinates as $(a_1,\ldots,a_m,b,c)$. Let
  $$\mathcal H\doteq\{x \in \mathbb C^{m+2}:x[b]=x[c]=0,\ \sum_{i=1}^m x[a_i]=0\}.$$
  If $x\in\mathcal H$, then $Qx=0$ and $e^\top x=0$, so $B_tx=x$. 
  Since $\dim(\mathcal H)=m-1$, the restriction $B_t|_{\mathcal H}$ is the identity map on $\mathcal H$, so this block contributes $m-1$ eigenvalues equal to $1$.
  Now consider the symmetric subspace
  $$\mathcal{U} \doteq \qty{x \in \mathbb C^{m+2}:x[a_1]=\cdots=x[a_m]}.$$
  Define the linear isomorphism
  $$\Phi:\mathbb C^3\to \mathcal U,\qquad \Phi(u,v,w)=(u,\ldots,u,v,w).$$
  Let $y=(u,v,w)^\top$. Then
  $Q\Phi y=\Phi(Cy),\qquad e^\top\Phi y=r^\top y,$
  where
  $$C=\begin{pmatrix}0&1&0\\0&0&1\\1&0&0\end{pmatrix},\qquad r^\top=\begin{pmatrix}m&1&1\end{pmatrix}.$$
  Moreover,
  $$d_\mu e^\top\Phi y=\Phi(\bar d r^\top y),\qquad \bar d=\begin{pmatrix}\varepsilon\\ \varepsilon\\ 1-(m+1)\varepsilon\end{pmatrix}.$$
  Hence
  $(Q-td_\mu e^\top)\Phi y=\Phi\qty((C-t\bar d r^\top)y).$
  In particular, $\mathcal U$ is invariant under $Q-td_\mu e^\top$, and therefore also under $B_t=I-Q+td_\mu e^\top$. The matrix of the restricted map $(Q-td_\mu e^\top)|_{\mathcal U}$ in the coordinates given by $\Phi$ is
  $M_t\doteq C-t\bar d r^\top.$
  Consequently, the restriction of $B_t$ to $\mathcal U$ is represented by $I-M_t$.
  Finally, $\mathbb C^{m+2}=\mathcal H\oplus\mathcal U$. Indeed, for any $x\in\mathbb C^{m+2}$, let $\bar u\doteq m^{-1}\sum_{i=1}^m x[a_i]$. Then
  $$x=\underbrace{(x[a_1]-\bar u,\ldots,x[a_m]-\bar u,0,0)}_{\in\mathcal H}+\underbrace{(\bar u,\ldots,\bar u,x[b],x[c])}_{\in\mathcal U},$$
  and $\mathcal H\cap\mathcal U=\{0\}$. Since both subspaces are $B_t$-invariant, the matrix of $B_t$ in a basis adapted to this direct sum is block diagonal:
  $$B_t\sim\begin{pmatrix}I_{m-1}&0\\0&I-M_t\end{pmatrix}.$$
  Thus the remaining three eigenvalues of $B_t$ are exactly the eigenvalues of $I-M_t$.
  Now compute the characteristic polynomial of $M_t$. Let $N=zI-C$. Then $\det N=z^3-1$ and
  $$\operatorname{adj}(N)=
    \begin{pmatrix}
      z^2&z&1\\
      1&z^2&z\\
      z&1&z^2
    \end{pmatrix}.$$
  By the matrix determinant lemma,
  $$\det(zI-M_t)=\det(N+t\bar d r^\top) =\det N+t r^\top\operatorname{adj}(N)\bar d.$$
  $$r^\top\operatorname{adj}(N)\bar d  =z^2+z+\left[m-\varepsilon(m^2+m-2)\right] =z^2+z+22.$$
  Therefore $\det(zI-M_t)=z^3-1+t(z^2+z+22)$. The Hurwitz polynomial of $I-M_t$ is
  $$q_t(s)=\det(sI+I-M_t)=s^3+(t+3)s^2+(3t+3)s+24t.$$

  For a cubic $s^3+as^2+bs+c$, Hurwitz stability is equivalent to $a>0$, $b>0$, $c>0$, and $ab>c$. Here $a,b,c>0$ for $t>0$, and the remaining condition is   $$(t+3)(3t+3)>24t \quad\Longleftrightarrow\quad 3(t-1)(t-3)>0.$$
  Thus the three-dimensional block is positive stable if and only if $t\in(0,1)\cup(3,\infty)$. Since the other $m-1$ eigenvalues of $B_t$ are equal to $1$, we conclude that $$A_\eta\text{ is positive stable} \quad\Longleftrightarrow\quad \eta\in(0,\varepsilon)\cup(3\varepsilon,\infty).$$
\end{proof}

Lemma~4.9 of \citet{blaser2026almost} guarantees positive stability for sufficiently small $\eta$.
Theorem~\ref{thm:maximal_stability_threshold} characterizes the maximal initial interval on which positive stability holds.
In Example~\ref{example:counterexample}, its endpoint is $\eta_*=\varepsilon$, while the full stability region is $(0,\varepsilon)\cup(3\varepsilon,\infty)$.
Thus, stability can return after it is lost, and $\eta_*$ is not a threshold beyond which every $\eta$ is unstable.
Lemma~4.9 of \citet{blaser2026almost} does not specify a numerical value of $\eta_0$ (only the existence), 
so we cannot further make more detailed statements about the numerical gap between $\eta_0$ and $\eta_*$.

We believe the spirit of this counterexample extends to other settings. Its three types of states, $a_i$, $b$, and $c$, have a simple structure. Under the target policy, the process usually stays at $c$, with rare excursions $c\to a_i\to b\to c$. Here $c$ can be thought of as a normal state of the system, $a_i$ can be thought of as the $i$-th possible rare failure of the system and $b$ can be thought of as a common recovery for every failure. Under the behavior policy, $c$ is visited far more often than each $a_i$ or $b$.

\section{Experiments}
We first provide a numerical validation of Example~\ref{example:counterexample}.
Figure~\ref{fig:eigenvalue-phase-portrait} shows the nontrivial eigenvalues of $A_\eta/\varepsilon$ as $\eta$ increases from $0$ to $4\varepsilon$. The conjugate pair starts in the open right half plane, crosses the imaginary axis at $\eta=\varepsilon$, remains in the open left half plane for $\eta\in(\varepsilon,3\varepsilon)$, and crosses back into the open right half plane at $\eta=3\varepsilon$.
\begin{figure}[h]
  \centering 
  \par\includegraphics[width=0.95\linewidth]{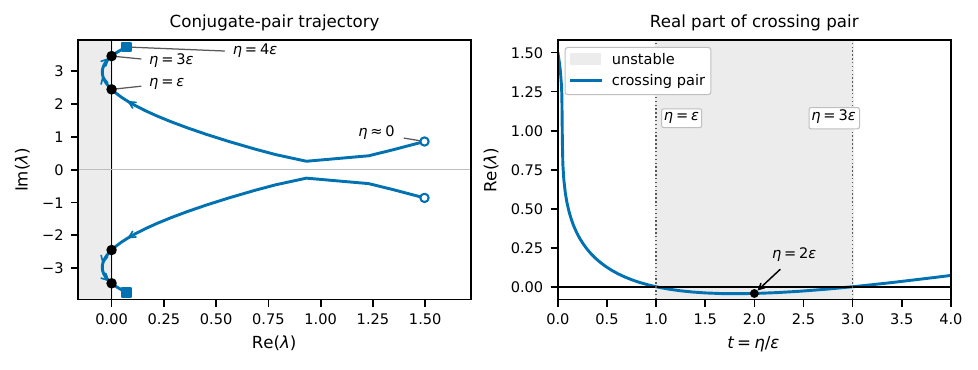}
  \caption{Phase portrait of the nontrivial eigenvalues of $A_\eta/\varepsilon$ for the counterexample in Example~\ref{example:counterexample}. Left: trajectory of the conjugate pair in the complex plane. Right: real part of this pair as a function of $t=\eta/\varepsilon$, with the unstable interval $t\in(1,3)$ shaded.}
  \label{fig:eigenvalue-phase-portrait}
\end{figure}

We then construct a concrete MDP instance based on Example~\ref{example:counterexample} and run the global-clock and local-clock differential TD algorithms on it with the same $\eta$. 
We consider a finite MDP with $m+2$ states and two actions $\qty{a_0, a_1}$, where $m>22$ is an integer.
The target policy $\pi$ always selects action $a_0$,
and the transition matrix of $a_0$ is defined to be $P_\pi$ in the proof of Example~\ref{example:counterexample}.
So the transition matrix of the target policy $\pi$ coincides with $P_\pi$.
The behavior policy $\mu$ selects action $a_0$ with probability $\kappa\in(0,1)$ and action $a_1$ otherwise, where $\kappa$ is chosen so that
$\kappa\le \min_{i,j:P_\pi[i,j]>0}\frac{d_\mu[j]}{P_\pi[i,j]},$
where both $P_\pi$ and $d_\mu$ are defined in the proof of Example~\ref{example:counterexample}.
Define the transition matrix of $a_1$ as $R\doteq(ed_\mu^\top-\kappa P_\pi)/(1-\kappa)$.
The matrix $R$ is a row stochastic matrix by the choice of $\kappa$ and the fact that $ed_\mu^\top$ and $P_\pi$ are row stochastic.
The transition matrix of the behavior policy is then $P_\mu=\kappa P_\pi+(1-\kappa)R=ed_\mu^\top.$
So $d_\mu$ coincides with the stationary distribution of the behavior policy. The reward is always 0. For the reported run, the initial state is fixed to be $a_1$.
We use the smallest instance in this family, $m=23$, and choose $\eta=2\varepsilon=\frac1{275}$, which lies in the unstable interval $(\varepsilon,3\varepsilon)$ for the global clock.
We then run the algorithm~\eqref{eq:differential_td global} with both the global clock and the local clock in the learning rates, using the base sequence $\alpha_n=0.45/(10^4+n)^{0.6}$.
The global-clock run uses $n=t+1$, while the local-clock run uses $n=\nu(S_t,t)$. We choose $v_0$ as the real part of an eigenvector corresponding to the eigenvalue of $A_\eta$ with the smallest real part, normalized so that $\norm{v_0}_2=1$, and set $\hat J_0=\eta e^\top v_0$. 
With this particular choice of $\hat J_0$, the results in \citet{wan2020learning} show that differential TD with a local clock should converge to 0 for any $\eta > 0$.
We run the algorithm for $10^{11}$ time steps and plot both $\norm{v_t}_2$ and $\operatorname{dist}(v_t,\operatorname{span}\{e\})$. Since the reward is always 0, the latter is the distance from $v_t$ to $\{\bar v_\pi+ce:c\in\mathbb R\}$. The result in Figure~\ref{fig:local-global-stability} shows that the local-clock recursion contracts to 0 while the global-clock recursion grows, matching the exact stability calculation above.
\begin{figure}[t]
  \centering \includegraphics[width=0.95\linewidth]{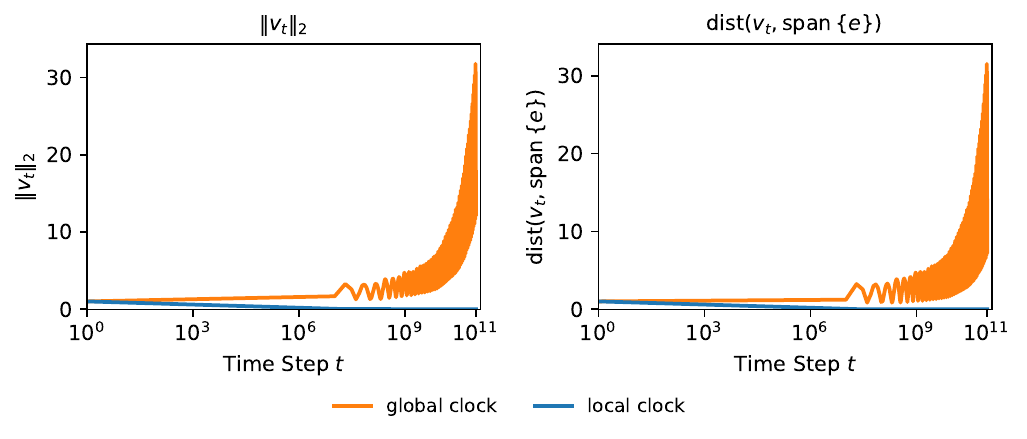}
  \caption{Differential TD with $\eta=2\varepsilon$. Left: $\norm{v_t}_2$. Right: distance from $v_t$ to $\operatorname{span}\{e\}$. The global-clock trajectory is oscillatory because the unstable mode is the conjugate pair shown in Figure~\ref{fig:eigenvalue-phase-portrait}. 
  The curve here is only a typical run with one seed, because we are interested in the sample path behavior. 
  We provide 9 more seeds in Appendix~\ref{app:additional_online_runs}, which show the same phenomenon.}
  \label{fig:local-global-stability}
\end{figure}
Additional experiments in Appendix~\ref{app:threshold_sweep} vary $\eta$ and show consistent algorithmic behavior with the stability region in~\eqref{eq:stability_region}.


\section{Related Work}
The closest to our work is \citet{chen2025non}, 
which also shows that an average-reward algorithm converges with a local clock but can diverge with a global clock.
However, there are several differences between this work and \citet{chen2025non}.
First, \citet{chen2025non} considers only the  base learning-rate sequence $\alpha_n = \fO(n^{-1})$.
(This rate describes the base sequence. With a local clock, the index is the corresponding state-action visit count.)
By using a local clock in this form of learning rate, \citet{chen2025non} implicitly approximates the state-action visitation distribution and thus implicitly applies importance sampling.
Their insights and analyses do not apply to learning rates that do not have this form, such as $\alpha_n = \fO(n^{-\beta})$ for some $\beta \in (1/2,1)$.
By contrast,
our analysis applies to any learning rate that satisfies the standard Robbins-Monro conditions,
including those with $\beta \in (1/2, 1)$.
Second, \citet{chen2025non} focuses on running standard discounted $Q$-learning with $\gamma = 1$ in the average-reward setting directly.
With the local clock implicitly applying importance sampling,
their algorithm converges to the \emph{set} of optimal action-value functions,
i.e.,
the $q$ value estimates themselves are not guaranteed to converge to a unique limit.
So \citet{chen2025non} has to introduce a new centered version for convergence to a unique limit. 
By contrast, the differential TD algorithm we analyze is designed for the average-reward setting and implicitly applies the reference state technique from RVI $Q$-learning, which directly guarantees convergence to a unique limit under the local-clock convergence assumptions in Appendix~\ref{app:convergence_assumptions}.
Nevertheless,
one may ask,
can the insights and techniques of \citet{chen2025non} be applied towards understanding the open problems in \citet{wan2020learning,blaser2026almost}?
Our evaluation is negative.
At most, \citet{chen2025non} can provide an alternative convergence proof for differential TD with a local clock and  base learning-rate sequence $\alpha_n = \fO(n^{-1})$. 
But combining the results from \citet{blaser2026almost} and this work,
it is easy to see that the effect of the global clock interacts with the hyperparameter $\eta$. 
But \citet{chen2025non} does not have the mechanism to make an evaluation based on $\eta$.

This work studies the divergence of differential TD through the lens of the instability of the associated ODE.
It is important to note that 
there are works showing that even if the ODE is not GAS, an algorithm can still converge in a weaker sense, e.g., convergence to a set of solutions \citep{xie2025finite,wang2024almost}, convergence to a bounded set \citep{meyn2024projected,liu2025linearq,liu2025extensions}, or convergence to a sample path dependent limit \citep{blaser2026asymptotic}.
However, as demonstrated by our experiments, the instability of the ODE we encounter does not appear benign.
That being said, proving that  differential TD with a global clock can diverge almost surely is still an open problem.

This work also sheds light on the deadly triad problem \citep{sutton2018reinforcement,zhang2022thesis} in RL, which arises from the interaction of function approximation, off-policy learning, and bootstrapping. 
There have been numerous studies for breaking the deadly triad in the discounted setting \citep{sutton2009convergent,sutton2009fast,sutton2016emphatic,zhang2022truncated,zhang2021breaking,zhang2020gradientdice,qian2025revisiting}.
Some of these approaches extend to the average-reward setting, e.g., 
\citet{zhang2021average} extend \citet{sutton2009convergent,sutton2009fast} to the average-reward setting.
But since a global clock is commonly used with function approximation,
future works on breaking the deadly triad in the average-reward setting should carefully investigate the interaction between the global clock and the average-reward estimation rate hyperparameter $\eta$.
Similarly, policy gradient and actor critic methods \citep{sutton1999policy,konda2000actor} are well studied in the discounted setting with a global clock \citep{agarwal2019optimality,mei2020global,zhang2022global}.
But future works extending the analysis to the average-reward setting should also carefully consider the interaction between the global clock and the average-reward estimation rate hyperparameter $\eta$.

\section{Conclusion}
This work shows that the choice of local and global clocks can affect the stability of average-reward RL algorithms in an arguably surprising way: for some $\eta > 0$, the ODE associated with differential TD can be stable with a local clock but unstable with a global clock,
a phenomenon that has never been reported in the discounted setting.
Our result shows that, in average-reward RL, replacing local clocks by the practically standard global clock is not a benign implementation detail.
A theorem proved with local clocks may be misleading for the algorithm people actually implement. A possible direction for future work is to extend the analysis to understand the behavior of differential $Q$-learning with a global clock and its dependence on the hyperparameter $\eta$.
We imagine that such an analysis would be very challenging given the nonlinear nature of the differential $Q$-learning updates.
However, we argue that our analysis for differential TD can serve as a foundation because the ODE associated with differential $Q$-learning can be regarded as a switched system (cf. \citet{lee2019unified}),
where each subsystem's stability can be analyzed similarly to the differential TD case.

This work also makes a contribution to the rank-one perturbation theory \citep{moro2003low, savchenko2004change, ding2007eigenvalues, mehl2011eigenvalue, ran2012eigenvalues, fourie2013rank, mehl2014eigenvalue, ran2021global}.
Particularly, \citet{bierkens2014singular} study the positive stability of rank-one perturbation to a singular $M$-matrix and establish a set of sufficient conditions.
One important result is regarding the positive stability of $B_\eta \doteq L + \eta uv^\top $,
where $L$ is a singular $M$-matrix with 0 being an algebraically simple eigenvalue and $u$ and $v$ are elementwise nonnegative.
Lemma 2.11 of \citet{bierkens2014singular} shows that $B_\eta$ is positive stable for sufficiently small $\eta > 0$.
The positive stability of $B_\eta$ for sufficiently large $\eta$ is left as Conjecture 2.17 in \citet{bierkens2014singular}.
Later on, \citet{anehila2022note} show by a counterexample that Conjecture 2.17 of \citet{bierkens2014singular} is false.
Our counterexample further shows that even if we impose additional conditions such as $u$ being strictly positive and $v$ being the all-one vector, the positive stability of $B_\eta$ does not necessarily hold for all $\eta > 0$.
Those counterexamples together testify to the hardness of characterizing the positive stability of rank-one perturbation to a singular $M$-matrix.
It remains an open problem whether there is some nontrivial sufficient condition that can guarantee the positive stability of $B_\eta$ for sufficiently large $\eta$.
Notably, as noted by \citet{bierkens2014singular,anehila2022note},
such counterexamples are hard to construct and random sampling is unlikely to find them.
Our experience is the same. 
We have tried to find counterexamples by random sampling and failed.
And we have to construct the counterexample manually based on the theoretical insights in Theorems~\ref{thm:maximal_stability_threshold} \&~\ref{thm:exact_stability_region}.

\newpage

\begin{ack}
  This work is supported in part by the US National Science Foundation under the awards III-2128019, SLES-2331904, and CAREER-2442098, the Commonwealth Cyber Initiative's Central Virginia Node under the awards VV-1Q26-001, VV-1Q27-007, and VV-1Q27005, 4VA at UVA, a Cisco Faculty Research Award, and a Google Research Award.
\end{ack}

\bibliographystyle{plainnat}
{\small
\bibliography{bibliography}
}


\newpage
\appendix

\section{Small-$\eta$ positive stability}
\label{app:small_eta_stability}

\begin{lemma}[Small-$\eta$ positive stability]
\label{lemma:small_eta_stability}
Suppose that $P_\pi$ is irreducible and $d_\mu(s)>0$ for every state $s$.
Then there exists $\eta_0>0$ such that
$$A_\eta=D_\mu(I-P_\pi+\eta ee^\top)$$
is positive stable for every $\eta\in(0,\eta_0]$.
\end{lemma}

\begin{proof}
  Recall that $L\doteq D_\mu(I-P_\pi)$, so $A_\eta=L+\eta d_\mu e^\top$. Lemma~\ref{lemma:spectrum} shows that every nonzero eigenvalue of $L$ has strictly positive real part. It remains to understand the zero eigenvalue of $L$.
  Since $D_\mu$ is invertible, $\ker(L)=\ker(I-P_\pi)$. Because $P_\pi$ is irreducible and row-stochastic, $\ker(I-P_\pi)=\operatorname{span}\{e\}$ by the Perron-Frobenius theorem.

  Let $d_\pi$ be the stationary distribution of $P_\pi$ and define $$\ell^\top\doteq d_\pi^\top D_\mu^{-1}.$$ Then $$\ell^\top L=d_\pi^\top(I-P_\pi)=0,\qquad \ell^\top e=d_\pi^\top D_\mu^{-1}e>0.$$
  The zero eigenvalue of $L$ is algebraically simple. If we suppose otherwise, we know that since its eigenspace is one-dimensional (because $\ker(L)=\operatorname{span}\{e\}$ is one-dimensional), there would exist a generalized eigenvector $y$ such that $Ly=e$. Multiplying by $\ell^\top$ yields $$0=\ell^\top Ly=\ell^\top e>0,$$ which is a contradiction.

  We now compute how this simple zero eigenvalue moves under the perturbation. Define
  $$F(x,\lambda,\eta)\doteq
    \begin{pmatrix}
      A_\eta x-\lambda x\\
      \ell^\top x-\ell^\top e
    \end{pmatrix}.$$
  At $(x,\lambda,\eta)=(e,0,0)$, we have $F(e,0,0)=0$. The derivative of $F$ with respect to $(x,\lambda)$ at this point maps $(y,\theta)$ to
  $$\begin{pmatrix}
      Ly-\theta e\\
      \ell^\top y
    \end{pmatrix}.$$
  This map is injective. If $Ly-\theta e=0$ and $\ell^\top y=0$, then multiplying the first equation by $\ell^\top$ gives $-\theta\ell^\top e=0$, so $\theta=0$. Then $Ly=0$, hence $y=ce$, and $\ell^\top y=0$ gives $c=0$. Since the domain and codomain have the same finite dimension, the map is invertible.

  By the finite-dimensional implicit function theorem, applied after identifying complex vectors with real vectors, there are differentiable functions $x(\eta)$ and $\lambda(\eta)$ near $0$ such that $$A_\eta x(\eta)=\lambda(\eta)x(\eta),\qquad x(0)=e,\qquad \lambda(0)=0.$$
  Differentiating at $\eta=0$ yields
  $$d_\mu e^\top e+Lx'(0)=\lambda'(0)e.$$
  Multiplying by $\ell^\top$ gives
  $$\lambda'(0) =\frac{\ell^\top d_\mu\, e^\top e}{\ell^\top e} =\frac{|\fS|}{d_\pi^\top D_\mu^{-1}e}>0,$$
  because $\ell^\top d_\mu=d_\pi^\top D_\mu^{-1}d_\mu=d_\pi^\top e=1$ and $e^\top e=|\fS|$. Therefore the zero eigenvalue of $L$ moves into the open right half plane for all sufficiently small positive $\eta$.

  Here, all other eigenvalues of $L$ already have strictly positive real part by Lemma~\ref{lemma:spectrum}, and they remain in the open right half plane for sufficiently small $\eta>0$ by continuity of the roots of the characteristic polynomial. Hence there exists $\eta_0>0$ such that $A_\eta$ is positive stable for every $\eta\in(0,\eta_0]$.
\end{proof}

\section{Additional Online TD Runs}
\label{app:additional_online_runs}

Figures~\ref{fig:additional-online-runs-norm} and~\ref{fig:additional-online-runs-distance} repeat the differential TD experiment from Figure~\ref{fig:local-global-stability} for nine additional random seeds.

\begin{figure}[t]
  \centering
  \includegraphics[width=\linewidth]{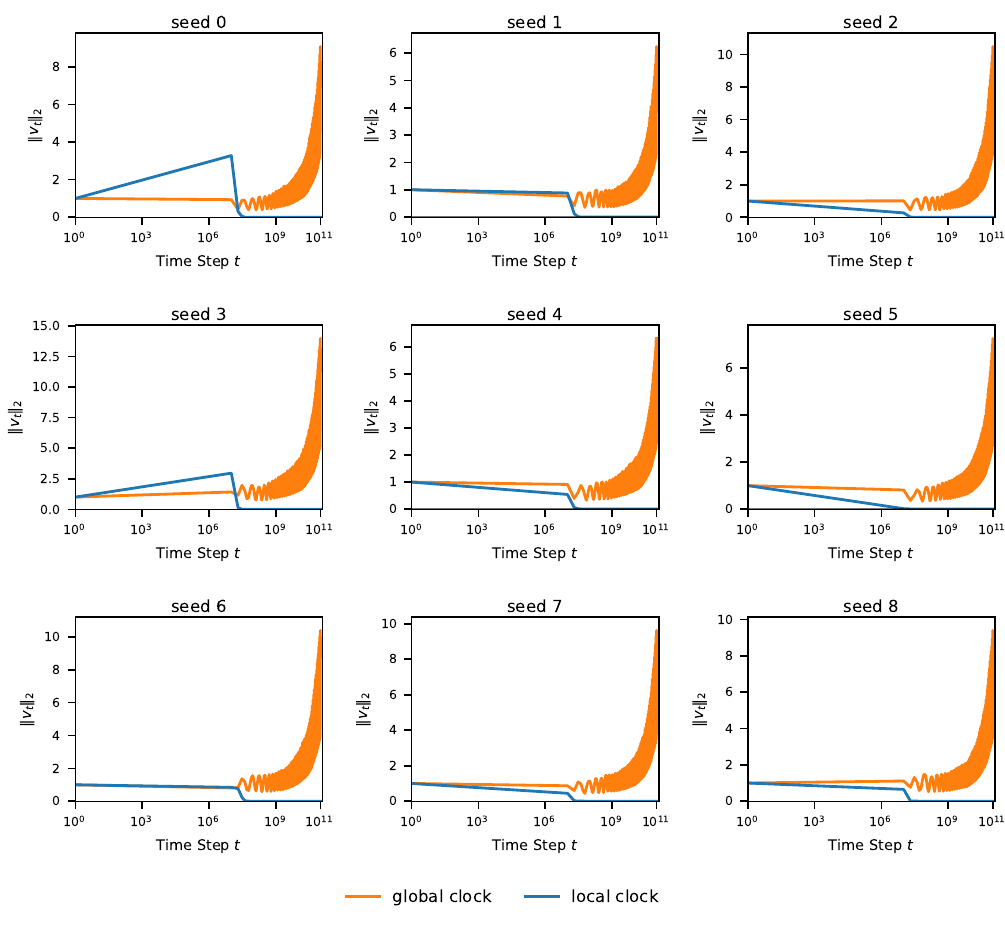}
  \caption{Additional differential TD runs with $\eta=2\varepsilon$, showing $\norm{v_t}_2$.}
  \label{fig:additional-online-runs-norm}
\end{figure}

\begin{figure}[t]
  \centering
  \includegraphics[width=\linewidth]{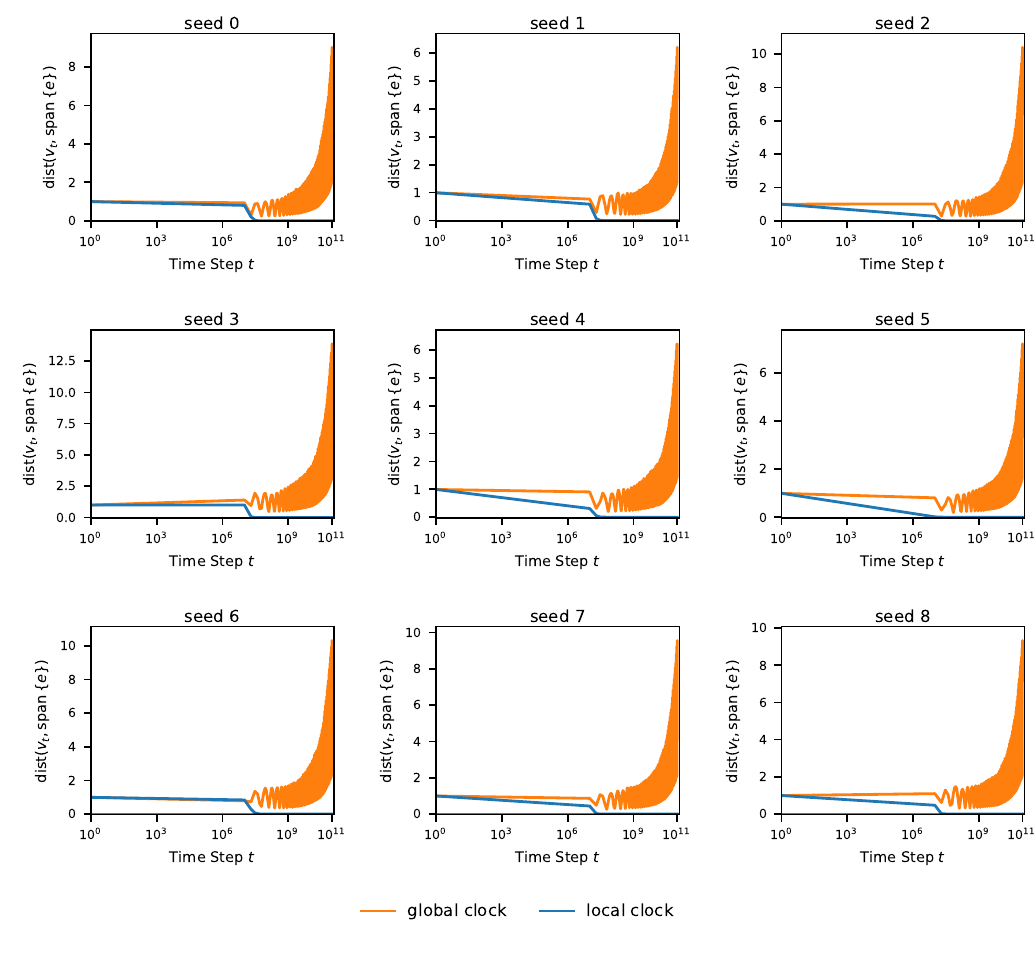}
  \caption{Additional differential TD runs with $\eta=2\varepsilon$, showing the distance from $v_t$ to $\operatorname{span}\{e\}$.}
  \label{fig:additional-online-runs-distance}
\end{figure}


\clearpage
\subsection{Threshold sweep}
\label{app:threshold_sweep}
We also run a threshold sweep for Example~\ref{example:counterexample} on both sides of $\eta=\varepsilon$ and $\eta=3\varepsilon$, with $m=23$ and $\varepsilon=1/550$.
We use $\kappa=\varepsilon$, the initial state $a_1$, the base sequence $\alpha_n=0.45/(10^4+n)^{0.6}$, and $T=10^9$ time steps.
For each $\eta$, we choose $v_0$ as the real part of an eigenvector corresponding to the eigenvalue of $A_\eta$ with the smallest real part, normalized so that $\norm{v_0}_2=1$, and set $\hat J_0=\eta e^\top v_0$.
All four settings use the same behavior trajectory (seed $9$).
Since the critical mode is a complex conjugate pair and the norm oscillates, Table~\ref{tab:threshold-sweep} reports the change over the second half of the run.
The endpoint changes are consistent with the stable--unstable--stable regions.
At the exact boundaries, the critical eigenvalues of the associated ODE are purely imaginary, so exponential divergence is not predicted.

\begin{table}[ht]
  \centering
  \small
  \captionsetup{position=top,skip=7pt,belowskip=0pt}
  \caption{Threshold sweep for global-clock differential TD. Here, $v_t^{\rm G}$ denotes the global-clock iterate and $\operatorname{dist}^{\rm G}_t=\operatorname{dist}(v_t^{\rm G},\operatorname{span}\{e\})$ is its Euclidean distance to the differential-value solution set.}
  \label{tab:threshold-sweep}
  \begin{tabular}{ccccc}
    \toprule
    $\eta/\varepsilon$ & ODE region & $\max_{\lambda\in\operatorname{spec}(-A_\eta)}\Re(\lambda)$ & $\|v^{\rm G}_{T/2}\|_2\to\|v^{\rm G}_{T}\|_2$ & $\operatorname{dist}^{\rm G}_{T/2}\to\operatorname{dist}^{\rm G}_{T}$ \\
    \midrule
    $0.85$ & stable & $-4.40\times10^{-5}$ & $2.500\to1.575$ & $2.191\to1.482$ \\
    $1.15$ & unstable & $+3.16\times10^{-5}$ & $2.647\to3.645$ & $2.607\to3.369$ \\
    $2.70$ & unstable & $+3.16\times10^{-5}$ & $3.375\to3.803$ & $3.318\to3.790$ \\
    $3.30$ & stable & $-3.61\times10^{-5}$ & $2.035\to0.980$ & $1.981\to0.693$ \\
    \bottomrule
  \end{tabular}
  \par\vspace{-8pt}
\end{table}

\clearpage
\section{Convergence assumptions}
\label{app:convergence_assumptions}
This appendix states the conditions used when invoking stochastic convergence. The matrix results in Section~\ref{sec:counterexamples} use Assumption~\ref{ass:matrix_conditions} and do not require a learning-rate sequence.

\paragraph{Local-clock differential TD.}
Theorem~B.3 of \citet{wan2020learning} invokes Assumptions~B.3, B.4, and B.6--B.8:
\begin{enumerate}
  \item The Markov chain induced by the target policy is unichain. The irreducibility assumption used here is sufficient.
  \item Coverage holds: $\pi(a\mid s)>0$ implies $\mu(a\mid s)>0$ for every state-action pair.
  \item The deterministic base learning-rate sequence satisfies
  \[
    \alpha_n>0,\qquad
    \sum_{n=0}^{\infty}\alpha_n=\infty,\qquad
    \sum_{n=0}^{\infty}\alpha_n^2<\infty.
  \]
  \item For every $x\in(0,1)$,
  \[
    \sup_n\frac{\alpha_{\lfloor xn\rfloor}}{\alpha_n}<\infty,
    \qquad
    \frac{\sum_{j=0}^{\lfloor yn\rfloor}\alpha_j}
         {\sum_{j=0}^{n}\alpha_j}\longrightarrow1
    \quad\text{uniformly for }y\in[x,1].
  \]
  \item There exists $\Delta>0$ such that every state satisfies
  \[
    \liminf_{n\to\infty}\frac{\nu(s,n)}{n+1}\geq\Delta
    \quad\text{almost surely}.
  \]
  Furthermore, for every $u>0$, define
  \[
    N(n,u)=\min\left\{m\geq n:\sum_{i=n+1}^{m}\alpha_i\geq u\right\}.
  \]
  For every pair of states $s,s'$, the limit
  \[
    \lim_{n\to\infty}
    \frac{\sum_{i=\nu(s,n)}^{\nu(s,N(n,u))}\alpha_i}
         {\sum_{i=\nu(s',n)}^{\nu(s',N(n,u))}\alpha_i}
  \]
  exists almost surely.
\end{enumerate}
Under these conditions, local-clock differential TD converges almost surely for every $\eta>0$.

\paragraph{Global-clock differential TD.}
We adopt the following sufficient conditions for stochastic convergence:
\begin{enumerate}
  \item The finite target and behavior chains are irreducible and aperiodic.
  \item Coverage holds: $\pi(a\mid s)>0$ implies $\mu(a\mid s)>0$ for every state-action pair.
  \item The deterministic base learning-rate sequence is positive and nonincreasing, with
  \[
    \alpha_n\to0,\qquad
    \sum_{n=0}^{\infty}\alpha_n=\infty,\qquad
    \sum_{n=0}^{\infty}\alpha_n^2<\infty,\qquad
    \frac{\alpha_n-\alpha_{n+1}}{\alpha_n}=O(\alpha_n).
  \]
  \item $A_\eta$ is positive stable. Lemma~4.9 of \citet{blaser2026almost} supplies this condition for $\eta\in(0,\eta_0]$ for some $\eta_0>0$.
\end{enumerate}
These are sufficient conditions adopted here; the square-summability condition is additional to the learning-rate conditions printed in Assumption~4.2 of \citet{blaser2026almost}.
For the finite MDP considered here, ergodic behavior-policy sampling and coverage give a finite ergodic transition-sample chain and bounded affine update coefficients. These satisfy Assumption~$6'$ of \citet{liu2025ode}. Together with the stated learning-rate conditions and positive stability, Corollary~8 of \citet{liu2025ode} gives almost-sure convergence to the unique equilibrium of the reduced affine mean ODE.

\paragraph{Discounted TD.}
The corresponding chain, coverage, and learning-rate conditions above also suffice for discounted tabular TD with the respective clock. The associated matrices $I-\gamma P_\pi$ and $D_\mu(I-\gamma P_\pi)$ are positive stable for $\gamma\in[0,1)$, as discussed in Section~\ref{sec:clock_odes}.


\end{document}